\documentclass[a4paper,12pt,reqno]{amsart}

\usepackage{graphicx,amssymb,datetime,float,MnSymbol,currfile,tikz}
\usepackage[round]{natbib}
\usetikzlibrary{arrows}
\usetikzlibrary{decorations.markings}

\newtheorem{theorem}{Theorem}
\newtheorem{lemma}{Lemma}

\def\ci{\!\perp\!}
\def\nci{\!\not\perp\!}
\def\ra{\rightarrow}
\def\la{\leftarrow}

\def\ot{\mathrel{\reflectbox{\ensuremath{\multimap}}}}
\def\to{\multimap}

\newcommand{\comments}[1]{}

\tikzset{tt/.style={decoration={
  markings,
  mark=at position .485 with {\arrow{>}},
  mark=at position .515 with {\arrow{<}}},postaction={decorate}}}

\begin{document}

\title[]{Unifying DAGs and UGs}

\author[]{Jose M. Pe\~{n}a\\
IDA, Link\"oping University, Sweden\\
jose.m.pena@liu.se}

\date{\currenttime, \ddmmyydate{\today}, \currfilename}

\begin{abstract}
We introduce a new class of graphical models that generalizes Lauritzen-Wermuth-Frydenberg chain graphs by relaxing the semi-directed acyclity constraint so that only directed cycles are forbidden. Moreover, up to two edges are allowed between any pair of nodes. Specifically, we present local, pairwise and global Markov properties for the new graphical models and prove their equivalence. We also present an equivalent factorization property. Finally, we present a causal interpretation of the new models.
\end{abstract}

\maketitle

\section{INTRODUCTION}

Lauritzen-Wermuth-Frydenberg chain graphs (LWF CGs) are usually described as unifying directed acyclic graphs (DAGs) and undirected graphs (UGs) \citep[p. 53]{Lauritzen1996}. However, this is arguable because the only constraint that DAGs and UGs jointly impose is the absence of directed cycles, whereas LWF CGs forbid semi-directed cycles which is a stronger constraint. Moreover, LWF CGs do not allow more than one edge between any pair of nodes. In this work, we consider graphs with directed and undirected edges but without directed cycles. The graphs can have up to two different edges between any pair of nodes. Therefore, our graphs truly unify DAGs and UGs. Hence, we call them UDAGs.

As we will see, UDAGs generalize LWF CGs. Two other such generalizations that can be found in the literature are reciprocal graphs (RGs) \citep{Koster1996} and acyclic graphs (AGs) \citep{LauritzenandSadeghi2017}. The main differences between UDAGs and these two classes of graphical models are the following. UDAGs are not a subclass of RGs because, unlike RGs, they do not constrain the semi-directed cycles allowed. UDAGs are a subclass of AGs. However, \citeauthor{LauritzenandSadeghi2017} define a global Markov property for AGs but no local or pairwise Markov property. We define the three properties for UDAGs. \citeauthor{LauritzenandSadeghi2017} do define though a pairwise Markov property for a subclass of AGs called chain mixed graphs (CMGs), but no local Markov property. Moreover, UDAGs are not a subclass of CMGs because, unlike CMGs, they can have semi-directed cycles. In addition to the local, pairwise and global Markov properties, we also define a factorization property for UDAGs. Such a property exists for RGs but not yet for AGs or CMGs. We also note that the algorithm developed by \cite{Sonntagetal.2015} for learning LWF CGs from data can easily be adapted to learn UDAGs (see Appendix A). To our knowledge, there is no algorithm for learning RGs, AGs or CMGs. Finally, it is worth mentioning that our work complements that by \citet{Richardson2003}, where DAGs and covariance (bidirected) graphs are unified.

The rest of the paper is organized as follows. Section \ref{sec:preliminaries} introduces some notation and definitions. Sections \ref{sec:global} and \ref{sec:local} present the global, local and pairwise Markov properties for UDAGs and prove their equivalence. Section \ref{sec:factorization} does the same for the factorization property. Section \ref{sec:causal} presents a causal interpretation of UDAGs (which inspires a learning algorithm that can be found in Appendix B). Section \ref{sec:discussion} closes the paper with some discussion.

\section{PRELIMINARIES}\label{sec:preliminaries}

In this section, we introduce some concepts about graphical models. Unless otherwise stated, all the graphs and probability distributions in this paper are defined over a finite set of random variables $V$. The elements of $V$ are not distinguished from singletons. An UDAG $G$ is a graph with possibly directed and undirected edges but without directed cycles, i.e. $A \ra \cdots \ra A$ is forbidden. There may be up to two different edges between any pair of nodes. Edges between a node and itself are not allowed. We denote by $A \to B$ that the edge $A \ra B$ or $A - B$ or both are in $G$.

Given an UDAG $G$, the parents of a set $X \subseteq V$ are $pa(X) = \{B | B \ra A$ is in $G$ with $A \in X \}$. The children of $X$ are $ch(X) = \{B | A \ra B$ is in $G$ with $A \in X \}$. The neighbors of $X$ are $ne(X) = \{B | A - B$ is in $G$ with $A \in X \}$. The ancestors of $X$ are $an(X) = \{B | B \to \cdots \to A$ is in $G$ with $A \in X \}$. Moreover, $X$ is called ancestral set if $X = an(X)$. The descendants of $X$ are $de(X) = \{B | A \to \cdots \to B$ is in $G$ with $A \in X \}$. The sets just defined are defined with respect to $G$. When they are defined with respect to another UDAG, this is indicated with a subscript.

Given an UDAG $G$, the moral graph of $G$ is the UG $G^m$ such that $A - B$ is in $G^m$ if and only if $A \to B$, $A \ra C \la B$, or $A \ra C - \cdots - D \la B$ is in $G$. Given a set $W \subseteq V$, we let $G_W$ denote the subgraph of $G$ induced by $W$. Given an UG $H$, we let $H^W$ denote the marginal subgraph of $H$ over $W$, i.e. the edge $A - B$ is in $H^W$ if and only if $A - B$ is in $H$ or $A - V_1 - \cdots - V_n - B$ is $H$ with $V_1, \ldots, V_n \notin W$. A set of nodes of $H$ is complete if there exists an undirected edge between every pair of nodes in the set. A clique of $H$ is a maximal complete set of nodes. The cliques of $H$ are denoted as $cl(H)$.

A route between two nodes $V_1$ and $V_n$ of an UDAG $G$ is a sequence of (not necessarily distinct) edges $E_1, \ldots, E_{n-1}$ in $G$ such that $E_i$ links the nodes $V_i$ and $V_{i+1}$. A route is called a path if the nodes in the route are all different. An undirected route is a route whose edges are all undirected. A section of a route $\rho$ is a maximal undirected subroute of $\rho$. A section $V_{2} - \cdots - V_{n-1}$ of $\rho$ is called collider section if $V_{1} \ra V_{2} - \ldots - V_{n-1} \la V_{n}$ is a subroute of $\rho$. Given a set $Z \subseteq V$, $\rho$ is said to be $Z$-active if (i) every collider section of $\rho$ has a node in $Z$, and (ii) every non-collider section of $\rho$ has no node in $Z$.

Let $X$, $Y$, $W$ and $Z$ be disjoint subsets of $V$. We represent by $X \ci_p Y | Z$ that $X$ and $Y$ are conditionally independent given $Z$ in a probability distribution $p$. Every probability distribution $p$ satisfies the following four properties: Symmetry $X \ci_p Y | Z \Rightarrow Y \ci_p X | Z$, decomposition $X \ci_p Y \cup W | Z \Rightarrow X \ci_p Y | Z$, weak union $X \ci_p Y \cup W | Z \Rightarrow X \ci_p Y | Z \cup W$, and contraction $X \ci_p Y | Z \cup W \land X \ci_p W | Z \Rightarrow X \ci_p Y \cup W | Z$. If $p$ is strictly positive, then it also satisfies the intersection property $X \ci_p Y | Z \cup W \land X \ci_p W | Z \cup Y \Rightarrow X \ci_p Y \cup W | Z$.

\section{GLOBAL MARKOV PROPERTY}\label{sec:global}

Given three disjoint sets $X, Y, Z \subseteq V$, we say that $X$ is separated from $Y$ given $Z$ in an UDAG $G$, denoted as $X \ci Y | Z$, if every path in $(G_{an(X \cup Y \cup Z)})^m$ between a node in $X$ and a node in $Y$ has a node in $Z$. As the theorem below proves, this is equivalent to saying that there is no route in $G$ between a node of $X$ and a node of $Y$ that is $Z$-active. Note that these separation criteria generalize those developed by \citet{Lauritzen1996} and \citet{Studeny1998} for LWF CGs.

\begin{theorem}
The two separation criteria for UDAGs in the paragraph above are equivalent.
\end{theorem}

\begin{proof}
Assume that there is a $Z$-active route $\rho$ in $G$ between $A \in X$ and $B \in Y$. Clearly, every node in a collider section is in $an(Z)$. Moreover, every node in a non-collider section is ancestor of $A$, $B$ or a node in a collider section, which implies that it is in $an(A \cup B \cup Z)$. Therefore, there is a route between $A$ and $B$ in $(G_{an(X \cup Y \cup Z)})^m$. Moreover, the route can be modified into a route $\varrho$ that circumvents $Z$ by noting that there is an edge $V_1 - V_n$ in $(G_{an(X \cup Y \cup Z)})^m$ whenever $V_{1} \ra V_{2} - \cdots - V_{n-1} \la V_{n}$ is a subroute of $\rho$. The route $\varrho$ can be converted into a path by removing loops.

Conversely, assume that there is a path $\rho$ in $(G_{an(X \cup Y \cup Z)})^m$ between $A \in X$ and $B \in Y$ that circumvents $Z$. Note that $\rho$ can be converted into a route $\varrho$ in $G$ as follows: If the edge $V_1 - V_n$ in $\rho$ was added to $(G_{an(X \cup Y \cup Z)})^m$ because $V_1 \to V_n$, $V_1 \la V_n$ or $V_{1} \ra V_{2} - \cdots - V_{n-1} \la V_{n}$ was in $G_{an(X \cup Y \cup Z)}$, then replace $V_1 - V_n$ with $V_1 \to V_n$, $V_1 \la V_n$ or $V_{1} \ra V_{2} - \cdots - V_{n-1} \la V_{n}$, respectively. Note that the non-collider sections of $\varrho$ have no node in $Z$ for $\rho$ to circumvent $Z$, whereas the collider sections of $\varrho$ have all their nodes in $an(X \cup Y \cup Z)$ by definition of $(G_{an(X \cup Y \cup Z)})^m$.

Note that we can assume without loss of generality that all the collider sections of $\varrho$ have some node in $an(Z)$ because, otherwise, if there is a collider section with no node in $an(Z)$ but with some node $C$ in $an(X)$ then there is a route $A' \ot \cdots \ot C$ with $A' \in X$ which can replace the subroute of $\varrho$ between $A$ and $C$. Likewise for $an(Y)$ and some $B' \in Y$.

Finally, note that every collider section $V_{1} \ra V_{2} - \cdots - V_{n-1} \la V_{n}$ of $\varrho$ that has no node in $Z$ must have a node $V_i$ in $an(Z) \setminus Z$ with $2 \leq i \leq n-1$, which implies that there is a route $V_i \to \cdots \to C$ where $C$ is the only node of the route that is in $Z$. Therefore, we can replace the collider section with $V_1 \ra V_2 - \cdots - V_i \to \cdots \to C \ot \cdots \ot V_i - \cdots - V_{n-1} \la V_n$. Repeating this step results in a $Z$-active route between a node in $X$ and a node in $Y$.
\end{proof}

We say that a probability distribution $p$ satisfies the global Markov property with respect to an UDAG $G$ if $X \ci_p Y | Z$ for all disjoint sets $X, Y, Z \subseteq V$ such that $X \ci Y | Z$. Note that two non-adjacent nodes in $G$ are not necessarily separated. For example, $C \ci D | Z$ does not hold for any $Z \subseteq \{A, B, E, F, H\}$ in the UDAG in Figure \ref{fig:example}. This drawback is shared by AGs. Although this problem cannot be solved for general AGs \citep[Figure 6]{LauritzenandSadeghi2017}, it can be solved for the subclass of CMGs by adding edges without altering the separations represented \citep[Corollary 3.1]{LauritzenandSadeghi2017}. Unfortunately, a similar solution does not exist for UDAGs. For example, adding the edge $C \to D$ to the UDAG in Figure \ref{fig:example} makes $A \ci B | D$ cease holding, whereas adding the edge $C \la D$ makes $A \ci B | C \cup F$ cease holding. Adding two edges between $C$ and $D$ does not help either, since one of them must be $C - D$. The following lemma characterizes the problematic pairs of nodes.

\begin{figure}
\centering
\begin{tabular}{c}
\begin{tikzpicture}[inner sep=1mm]
\node at (-1,0) (A) {$A$};
\node at (3,0) (B) {$B$};
\node at (0,0) (C) {$C$};
\node at (2,0) (D) {$D$};
\node at (0,-1) (E) {$E$};
\node at (1,-1) (F) {$F$};
\node at (2,-1) (H) {$H$};

\path[->] (A) edge (C);
\path[->] (B) edge (D);
\path[->] (C) edge (E);
\path[->] (D) edge (H);
\path[-] (E) edge (F);
\path[-] (H) edge (F);
\path[->] (F) edge (C);
\end{tikzpicture}
\end{tabular}\caption{Example where non-adjacency does not imply separation.}\label{fig:example}
\end{figure}
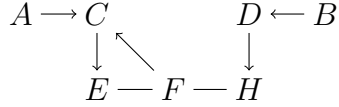

\begin{lemma}
Given two non-adjacent nodes $V_1$ and $V_n$ in an UDAG $G$, $V_1 \ci V_n | Z$ does not hold for any $Z \subseteq V \setminus ( V_1 \cup V_n )$ if and only if $V_{1} \ra V_{2} - \cdots - V_{n-1} \la V_{n}$ is in $G$, and $V_i \in an(V_1 \cup V_n)$ for some $1 < i < n$.\footnote{In the terminology of \citet{LauritzenandSadeghi2017}, this route is a primitive inducing walk.}
\end{lemma}

\begin{proof}
To prove the if part, assume without loss of generality that $V_i \in an(V_1)$. This together with the route in the lemma imply that $G$ has a route $\rho$ of the form $V_1 \ot \cdots \ot V_i - \cdots - V_{n-1} \la V_{n}$. If no node in $Z$ is in $\rho$, then $V_1 \ci V_n | Z$ does not hold due to $\rho$. If $C \in Z$ is in the subroute $V_i - \cdots - V_{n-1} \la V_{n}$ of $\rho$, then $V_1 \ci V_n | Z$ does not hold due to the route in the lemma. Finally, if $C \in Z$ is in the subroute $V_1 \ot \cdots \ot V_i$ of $\rho$, then $V_1 \ci V_n | Z$ does not hold due to the route $V_1 \ra V_2 - \cdots - V_i \to \cdots \to C \ot \cdots \ot V_i - \cdots - V_{n-1} \la V_n$.

To prove the only if part, simply consider $Z=\emptyset$ and note that $V_1$ and $V_n$ are adjacent in $(G_{an(V_1 \cup V_n)})^m$ only if $G$ has a subgraph of the form described in the lemma.
\end{proof}

\begin{figure}
\centering
\begin{tabular}{c}
\begin{tikzpicture}[inner sep=1mm]
\node at (0,0) (A) {$A$};
\node at (1,0) (B) {$B$};
\node at (2,0) (C) {$C$};
\node at (3,0) (D) {$D$};
\node at (1,-1) (E) {$E$};
\node at (2,-1) (F) {$F$};
\node at (1,-2) (I) {$I$};
\node at (2,-2) (J) {$J$};
\node at (0,-3) (K) {$K$};
\node at (1,-3) (L) {$L$};
\node at (2,-3) (M) {$M$};
\node at (3,-3) (N) {$N$};

\path[->] (A) edge (E);
\path[->] (B) edge (E);
\path[->] (C) edge (F);
\path[->] (D) edge (F);
\path[-] (E) edge (F);
\path[-] (I) edge (J);
\path[->] (E) edge (I);
\path[->] (J) edge (F);
\path[->] (K) edge (I);
\path[->] (L) edge (I);
\path[->] (M) edge (J);
\path[->] (N) edge (J);
\end{tikzpicture}
\end{tabular}\caption{Example of UDAG without Markov equivalent LWF CG.}\label{fig:examplesupplement}
\end{figure}
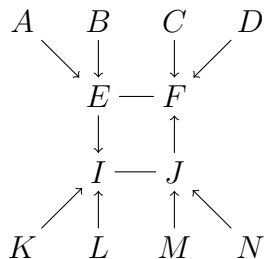

Finally, we show that the independence models representable with UDAGs are a proper superset of those representable with LWF CGs. In particular, we show that there is no LWF CG that is Markov equivalent to the UDAG in Figure \ref{fig:examplesupplement}, i.e. there is no LWF CG that represents exactly the independence model represented by the UDAG. Assume to the contrary that there is a LWF CG $H$ that is Markov equivalent to the UDAG $G$ in the figure. First, note that $A \ci B | \emptyset$ and $A \nci B | E$ imply that $H$ must have an induced subgraph $A \ra E \la B$. Likewise, $H$ must have induced subgraphs $C \ra F \la D$, $K \ra I \la L$, and $M \ra J \la N$. Next, note that $A \ci I | \{E, J\}$ implies that $H$ cannot have an edge $E \la I$. Likewise, $H$ cannot have an edge $F \ra J$. Note also that $A \ci F | \{B,C,D,E,J\}$ and $D \ci E | \{A,B,C,F,J\}$ imply that $H$ cannot have an edge $E \la F$ or $E \ra F$. Likewise, $H$ cannot have an edge $I \la J$ or $I \ra J$. Consequently, $H$ must have a subgraph of the form
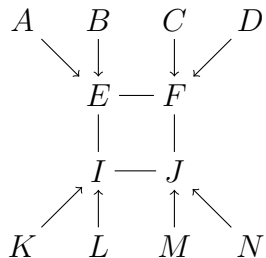
\begin{center}
\begin{tikzpicture}[inner sep=1mm]
\node at (0,0) (A) {$A$};
\node at (1,0) (B) {$B$};
\node at (2,0) (C) {$C$};
\node at (3,0) (D) {$D$};
\node at (1,-1) (E) {$E$};
\node at (2,-1) (F) {$F$};
\node at (1,-2) (I) {$I$};
\node at (2,-2) (J) {$J$};
\node at (0,-3) (K) {$K$};
\node at (1,-3) (L) {$L$};
\node at (2,-3) (M) {$M$};
\node at (3,-3) (N) {$N$};

\path[->] (A) edge (E);
\path[->] (B) edge (E);
\path[->] (C) edge (F);
\path[->] (D) edge (F);
\path[-] (E) edge (F);
\path[-] (I) edge (J);
\path[-] (E) edge (I);
\path[-] (J) edge (F);
\path[->] (K) edge (I);
\path[->] (L) edge (I);
\path[->] (M) edge (J);
\path[->] (N) edge (J);
\end{tikzpicture}
\end{center}
This implies that $A \ci N | \{B,C,D,E,F,I,J,K,L,M\}$ holds in $G$ but not in $H$, which is a contradiction.

The following lemma shows that the existence of a semi-directed cycle is not sufficient to declare an UDAG non-equivalent to any LWF CG. Instead, the semi-directed cycle must occur in a particular configuration, e.g. as in Figure \ref{fig:examplesupplement}. For instance, the lemma implies that the UDAG $A \ra B \la C - B$ is Markov equivalent to the LWF CG $A - B - C - A$.

\begin{lemma}\label{lem:m}
Let $G$ denote an UDAG. If (i) $W$ is an ancestral set of nodes in $G$ of size greater than one, and (ii) $W$ is minimal with respect to the property (i), then replacing $G_W$ with $(G_W)^m$ in $G$ results in an UDAG $H$ that is Markov equivalent to $G$.
\end{lemma}

\begin{proof}

First, note that $H$ is an UDAG because no directed cycle can be created by replacing $G_W$ with $(G_W)^m$ in $G$. Now, consider checking whether a separation $X \ci Y | Z$ holds in $G$ and $H$. Consider the following cases.

\begin{itemize}

\item[(1.)] Assume that $an(X \cup Y \cup Z)$ in $G$ includes no node in $W$. Then, $(G_{an(X \cup Y \cup Z)})^m = (H_{an(X \cup Y \cup Z)})^m$ and thus $X \ci Y | Z$ holds in both $G$ and $H$ or in none.

\item[(2.)] Assume that $an(X \cup Y \cup Z)$ in $G$ includes exactly one of the nodes in $W$, here denoted by $A$. Then, $an(X \cup Y \cup Z)$ in $H$ includes all the nodes in $W$ because $W$ is connected in $G$ since, otherwise, it is not minimal which is a contradiction. Moreover, note that $A$ is the only node shared by $(G_{an(X \cup Y \cup Z)})^m$ and $(G_W)^m$ because, otherwise, there must be a second node in $W$ that is in $an(X \cup Y \cup Z)$ in $G$, which is a contradiction. Then, $(H_{an(X \cup Y \cup Z)})^m = (G_{an(X \cup Y \cup Z)})^m \cup (G_W)^m$ and thus $X \ci Y | Z$ holds in both $G$ and $H$ or in none.

\item[(3.)] Assume that $an(X \cup Y \cup Z)$ in $G$ includes more than one of the nodes in $W$. Then, $(G_{an(X \cup Y \cup Z)})^m$ includes all the nodes in $W$ because, otherwise, $W$ is not minimal which is a contradiction. Then, $(G_{an(X \cup Y \cup Z)})^m = (H_{an(X \cup Y \cup Z)})^m$ and thus $X \ci Y | Z$ holds in both $G$ and $H$ or in none. To see it, note that by construction $(G_{an(X \cup Y \cup Z)})^m$ and $(H_{an(X \cup Y \cup Z)})^m$ differ only if the former has an edge $V_1 - V_n$ that the latter does not have. This occurs only if $G_{an(X \cup Y \cup Z)}$ has a subgraph of the form $V_1 \ra V_2 - \cdots - V_{n-1} \la V_n$, whereas $H_{an(X \cup Y \cup Z)}$ has a subgraph of the form $V_1 - V_2 - \cdots - V_{n-1} \la V_n$ or $V_1 - V_2 - \cdots - V_{n-1} - V_n$. The former case implies that $V_1, \ldots, V_{n-1}$ are in $W$ whereas $V_n$ is not. This contradicts the fact that $W$ is an ancestral set. The latter case implies that $V_1, \ldots, V_n$ are in $W$, which implies that $V_1 - V_n$ is in $H$, which is a contradiction.

\end{itemize}

\end{proof}

Note that the condition in the lemma is sufficient but not necessary: The UDAGs $A \ra B \ra C$ and $A \ra B - C$ are Markov equivalent.

\subsection{SEPARATION ALGORITHM}\label{sec:algorithm}

Since there may be infinite many routes in an UDAG $G$, one may wonder if the separation criterion based on ruling out $Z$-active routes that we have presented above is of any use in practice. The algorithm below shows how to implement it to check in a finite number of steps whether $X \ci Y | Z$ holds. The algorithm is a generalization of the one developed by \citet{Studeny1998} for LWF CGs, which was later slightly improved by \citet{Sonntagetal.2015}. The algorithm basically consists in repeatedly executing some rules to build the sets $U_1, U_2, U_3 \subseteq V$, which can be described as follows.
\begin{itemize}
\item $B \in U_1$ if and only if there exists a $Z$-active route between $A \in X$ and $B$ in $G$ which ends with the subroute $V_i \ra V_{i+1} - \cdots - V_{i+k}=B$ with $k\geq 1$.

\item $B \in U_2$ if and only if there exists a $Z$-active route between $A \in X$ and $B$ in $G$ which does not end with the subroute $V_i \ra V_{i+1} - \cdots - V_{i+k}=B$ with $k\geq 1$.

\item $B \in U_3$ if and only if there exists a node $C \in U_1 \cup U_2$ and a route $C=V_1 \rightarrow V_2 - \cdots - V_k = B$ in $G$ with $k \geq 2$ such that $\{V_2,\ldots,V_k\} \cap Z \neq \emptyset$.
\end{itemize}

The algorithm starts with $U_1=U_3=\emptyset$ and $U_2=X$. The algorithm executes the following rules until $U_1$, $U_2$ and $U_3$ cannot be further enlarged.
\begin{itemize}
\item $C \in U_2$, $C \ot D$ is in $G$, and $D \notin Z \Rightarrow D \in U_2$.
\item $C \in U_1 \cup U_2$, $C \ra D$ is in $G$, and $D \notin Z \Rightarrow D \in U_1$.
\item $C \in U_1$, $C - D$ is in $G$, and $D \notin Z \Rightarrow D \in U_1$.
\item $C \in U_1 \cup U_2$, $C \ra D$ is in $G$, and $D \in Z \Rightarrow D \in U_3$.
\item $C \in U_1$, $C - D$ is in $G$, and $D \in Z \Rightarrow D \in U_3$.
\item $C \in U_3$, and $C - D$ is in $G \Rightarrow D \in U_3$.
\item $C \in U_3$, $C \la D$ is in $G$, and $D \notin Z \Rightarrow D \in U_2$.
\end{itemize}

One can prove that, when the algorithm halts, there is a $Z$-active route between each node in $U_1 \cup U_2$ and some node in $X$. The proof is identical to the one for LWF CGs by \citet[Lemma 5.2]{Studeny1998} and \citet[Proposition 1]{Sonntagetal.2015}. Therefore, $X \ci Y | Z$ if and only if $Y \subseteq V \setminus ( U_1 \cup U_2 )$. In Appendix A, we use this result to develop an algorithm for learning UDAGs from data.

\section{LOCAL AND PAIRWISE MARKOV PROPERTIES}\label{sec:local}

We say that a probability distribution $p$ satisfies the local Markov property with respect to an UDAG $G$ if for any ancestral set $W$,
\[
A \ci_p W \setminus (A \cup ne_{(G_W)^m}(A)) | ne_{(G_W)^m}(A)
\]
for any $A \in W$. Similarly, we say that a probability distribution $p$ satisfies the pairwise Markov property with respect to $G$ if for any ancestral set $W$,
\[
A \ci_p B | W \setminus (A \cup B)
\]
for any $A, B \in W$ such that $B \notin ne_{(G_W)^m}(A)$.

\begin{theorem}\label{the:eqLP}
Given a probability distribution $p$ satisfying the intersection property, $p$ satisfies the local Markov property with respect to an UDAG $G$ if and only if it satisfies the pairwise Markov property with respect to $G$.
\end{theorem}

\begin{proof}
The if part follows by repeated application of the intersection property. The only if part follows by the weak union property.
\end{proof}

\begin{theorem}\label{the:eqPG}
Given a probability distribution $p$ satisfying the intersection property, $p$ satisfies the pairwise Markov property with respect to an UDAG $G$ if and only if it satisfies the global Markov property with respect to $G$.
\end{theorem}

\begin{proof}
The if part is trivial. To prove the only if part, let $W=an(X \cup Y \cup Z)$ and note that the pairwise and global Markov properties are equivalent for UGs \citep[Theorem 3.7]{Lauritzen1996}.
\end{proof}

Note that the local Markov property for LWF CGs specifies a single independence for each node \citep[p. 55]{Lauritzen1996}. However, the local Markov property for UDAGs specifies many more independences, specifically an independence for any node and ancestral set containing the node. All in all, our local Markov property serves its purpose, namely to identify a subset of the independences specified by the global Markov property that implies the rest. In the next section, we show how to reduce this subset.

\subsection{REDUCTION}\label{sec:reductionlocal}

The number of independences specified by the local Markov property for UDAGs can be reduced by considering only maximal ancestral sets for any node $A$, i.e. those ancestral sets $W'$ such that $A \in W'$ and $ne_{(G_{W'})^m}(A) \subset ne_{(G_{W''})^m}(A)$ for any ancestral set $W''$ such that $W' \subset W''$. Note that there may be several maximal ancestral sets $W'$ for $A$, each for a different set $ne_{(G_{W'})^m}(A)$ as will be shown. The independences for the non-maximal ancestral sets follow from the independences for the maximal ancestral sets by the decomposition property. In other words, for any non-maximal ancestral set $W$ and $A \in W$,
\[
A \ci_p W \setminus (A \cup ne_{(G_W)^m}(A)) | ne_{(G_W)^m}(A)
\]
follows from
\[
A \ci_p W' \setminus (A \cup ne_{(G_{W'})^m}(A)) | ne_{(G_{W'})^m}(A)
\]
where $W'$ is the maximal ancestral set for $A$ such that $ne_{(G_W)^m}(A) = ne_{(G_{W'})^m}(A)$. In the UDAG in Figure \ref{fig:example2}, for instance, $W_1=\{A, B, C, D \}$, $W_2=\{ A, B, C, D, E, I, J, K \}$, and $W_3=\{A, B, C, D, E, F, H, I, J, K \}$ are three ancestral sets that contain the node $B$. However, only $W_2$ and $W_3$ are maximal for $B$: $W_1$ is not maximal because $W_1 \subset W_2$ but $ne_{(G_{W_1})^m}(B) = ne_{(G_{W_2})^m}(B)$, and $W_2$ is maximal because $W_2 \subset W_3$ and $ne_{(G_{W_2})^m}(B) \subset ne_{(G_{W_3})^m}(B)$. Note that $W_1$ specifies $B \ci_p D | \{A, C\}$, and $W_2$ specifies $B \ci_p \{D, E, I, J, K\} | \{A, C\}$. Clearly, the latter independence implies the former by the decomposition property. Therefore, there is no need to specify both independences, as the local Markov property does. It suffices to specify just the second.

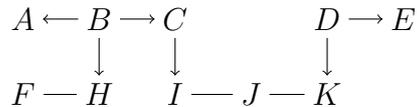
\begin{figure}
\centering
\begin{tabular}{c}
\begin{tikzpicture}[inner sep=1mm]
\node at (0,0) (A) {$A$};
\node at (1,0) (B) {$B$};
\node at (2,0) (C) {$C$};
\node at (4,0) (D) {$D$};
\node at (5,0) (E) {$E$};
\node at (0,-1) (F) {$F$};
\node at (1,-1) (H) {$H$};
\node at (2,-1) (I) {$I$};
\node at (3,-1) (J) {$J$};
\node at (4,-1) (K) {$K$};

\path[->] (B) edge (A);
\path[->] (B) edge (C);
\path[->] (B) edge (H);
\path[->] (D) edge (E);
\path[->] (D) edge (K);
\path[->] (C) edge (I);
\path[-] (I) edge (J);
\path[-] (J) edge (K);
\path[-] (F) edge (H);
\end{tikzpicture}
\end{tabular}\caption{Example where the local Markov property can be improved by considering only maximal ancestral sets.}\label{fig:example2}
\end{figure}

A more convenient characterization of maximal ancestral sets is the following. An ancestral set $W'$ is maximal for $A \in W'$ if and only if $W' = V \setminus [ ( ch(A) \cup de(ch(A)) ) \setminus W' ]$. To see it, note that $B \in ne_{(G_{W'})^m}(A)$ if and only if $A \ot B$, $A \ra B$, or $A \ra C - \cdots - D \la B$ is in $G_{W'}$. Note that all the parents and neighbors of $A$ are in $W'$, because $W'$ is ancestral. However, if there is some child $B$ of $A$ that is not in $W'$, then any ancestral set $W''$ that contains $W'$ and $B$ or any node that is a descendant of $B$ will be such that $ne_{(G_{W'})^m}(A) \subset ne_{(G_{W''})^m}(A)$.

The number of independences specified by the pairwise Markov property can also be reduced by considering only maximal ancestral sets. This can be proven in the same way as Theorem \ref{the:eqLP}.

\section{FACTORIZATION PROPERTY}\label{sec:factorization}

\begin{theorem}\label{the:factorization}
Given a probability distribution $p$ satisfying the intersection property, $p$ satisfies the pairwise Markov property with respect to an UDAG $G$ if and only if for any ancestral set $W$,
\[
p(W) = \prod_{K \in cl((G_W)^m)} \varphi(K)
\]
where $\varphi(K)$ is a non-negative function.
\end{theorem}

\begin{proof}
It suffices to recall the equivalence between the pairwise Markov property and the factorization property for UGs \citep[Theorem 3.9]{Lauritzen1996}.
\end{proof}

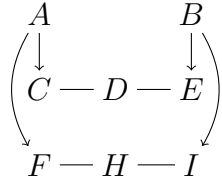
\begin{figure}
\centering
\begin{tabular}{c}
\begin{tikzpicture}[inner sep=1mm]
\node at (0,0) (A) {$A$};
\node at (2,0) (B) {$B$};
\node at (0,-1) (C) {$C$};
\node at (1,-1) (D) {$D$};
\node at (2,-1) (E) {$E$};
\node at (0,-2) (F) {$F$};
\node at (1,-2) (H) {$H$};
\node at (2,-2) (I) {$I$};

\path[->] (A) edge (C);
\path[->] (B) edge (E);
\path[->] (A) edge [bend right] (F);
\path[->] (B) edge [bend left] (I);
\path[-] (C) edge (D);
\path[-] (E) edge (D);
\path[-] (F) edge (H);
\path[-] (I) edge (H);
\end{tikzpicture}
\end{tabular}\caption{Example where the factorization property can be improved by considering only maximal ancestral sets.}\label{fig:example3}
\end{figure}

\subsection{REDUCTION}

The number of factorizations specified by the factorization property for UDAGs can be reduced by considering only maximal ancestral sets, i.e. those ancestral sets $W'$ such that $(G_{W'})^m$ is a proper subgraph of $((G_{W''})^m)^{W'}$ for any ancestral set $W''$ such that $W' \subset W''$. These maximal ancestral sets do not necessarily coincide with the ones defined in Section \ref{sec:reductionlocal}. The factorizations for the non-maximal ancestral sets follow from the factorizations for the maximal ancestral sets. To see it, note that for any non-maximal ancestral set $W$, the probability distribution $p(W)$ can be computed by marginalization from $p(W')$ where $W'$ is any maximal ancestral set such that $((G_{W'})^m)^{W} = (G_{W})^m$. Note also that $p(W)$ factorizes according to $((G_{W'})^m)^{W}$ and thus according to $(G_{W})^m$, by \citet[Lemma 3.1]{Studeny1997} and \citet[Theorems 3.7 and 3.9]{Lauritzen1996}. In the UDAG in Figure \ref{fig:example3}, for instance, $W_1=\{A, B \}$, $W_2=\{ A, B, C, D, E \}$, and $W_3=\{A, B, C, D, E, F, H, I \}$ are three ancestral sets. However, only $W_1$ and $W_3$ are maximal: $W_2$ is not maximal because $W_2 \subset W_3$ but $((G_{W_3})^m)^{W_2} = (G_{W_2})^m$, and $W_1$ is maximal because $W_1 \subset W_3$ and $(G_{W_1})^m$ is a proper subgraph of $((G_{W_3})^m)^{W_1}$. Note that $W_3$ specifies 
\begin{align*}
p(W_3) =& \varphi_3(A,B) \varphi_3(A,C) \varphi_3(B,E) \varphi_3(C,D) \varphi_3(D,E)\\
&\cdot \varphi_3(A,F) \varphi_3(B,I) \varphi_3(F,H) \varphi_3(H,I)
\end{align*}
and $W_2$ specifies
\[
p(W_2) = \varphi_2(A,B) \varphi_2(A,C) \varphi_2(B,E) \varphi_2(C,D) \varphi_2(D,E).
\]
Clearly, the former factorization implies the latter by taking
\begin{align*}
\varphi_2(A,B) =& \varphi_3(A,B) \sum_{F,H,I} \varphi_3(A,F) \varphi_3(B,I) \varphi_3(F,H) \varphi_3(H,I)\\
\varphi_2(A,C)=&\varphi_3(A,C)\\
\varphi_2(B,E)=&\varphi_3(B,E)\\
\varphi_2(C,D)=&\varphi_3(C,D)\\
\varphi_2(D,E)=&\varphi_3(D,E).
\end{align*}
Therefore, there is no need to specify both factorizations, as the factorization property does. It suffices to specify just the first.

A more convenient characterization of maximal ancestral sets is the following. An ancestral set $W'$ is maximal if and only if $pa(A \cup an(A) \setminus W') \cap W'$ is not a complete set in $(G_{W'})^m$ for any node $A \in V \setminus W'$.\footnote{In the terminology of \citet{Frydenberg1990}, $A \cup an(A) \setminus W'$ is a non-simplicial set in $(G_{W'})^m$.} To see it, note that any ancestral set $W''$ that contains $W' \cup A$ will also contain $an(A) \setminus W'$. Note also that no node $B \in A \cup an(A) \setminus W'$ has a neighbor or child in $W'$ because, otherwise, $B \in W'$ which is a contradiction. So, any such node $B$ can only have parents in $W'$. Moreover, since $pa(A \cup an(A) \setminus W') \cap W'$ is not a complete set in $(G_{W'})^m$, there must be two nodes in $pa(A \cup an(A) \setminus W') \cap W'$ that are not adjacent in $(G_{W'})^m$. However, there is a path between these two nodes in $(G_{W''})^m$ through $B$, which implies that $(G_{W'})^m$ is a proper subset of $((G_{W''})^m)^{W'}$.

\section{CAUSAL INTERPRETATION}\label{sec:causal}

In this section, we propose a causal interpretation of UDAGs. We start by introducing some notation. Given an UDAG $G$, let $W_1, \ldots, W_t$ denote all the minimal ancestral sets in $G$. Assume that the sets are sorted such that if $W_i \subset W_j$ then $i < j$. Moreover, let $C_i = W_i \setminus \cup_{j<i} W_j$. Note that all the edges between a node in $C_i$ and a node in $C_j$ with $i<j$ are directed edges from $C_i$ to $C_j$. Note also that every node in $C_i$ is an ancestor of the rest of the nodes in $C_i$. Let $bd(C_i)=pa(C_i) \setminus C_i$. Moreover, let $(G_{C_i \cup bd(C_i)})^*$ be the result of adding undirected edges to $(G_{C_i \cup bd(C_i)})^m$ until $bd(C_i)$ is a complete set. Note that for LWF CGs, the sets $C_i$ correspond to the chain components, $bd(C_i)=pa(C_i)$, and $(G_{C_i \cup bd(C_i)})^*=(G_{C_i \cup bd(C_i)})^m$. For instance, in the UDAG in Figure \ref{fig:example} we have that $W_1=\{A\}$, $W_2=\{B\}$, $W_3=\{B,D\}$ and $W_4=\{A,B,C,D,E,F,H\}$, and $C_1=\{A\}$, $C_2=\{B\}$, $C_3=\{D\}$ and $C_4=\{C,E,F,H\}$, and $bd(C_1)=\emptyset$, $bd(C_2)=\emptyset$, $bd(C_3)=\{B\}$ and $bd(C_4)=\{A,D\}$.

The following theorem presents a factorization property for UDAGs. Compared to that in Theorem \ref{the:factorization}, it is simpler and resembles the factorization property for LWF CGs. However, it is necessary but not sufficient. It will be instrumental to derive our causal interpretation of UDAGs.

\begin{theorem}\label{the:factorization2}
Let $p$ be a probability distribution satisfying the intersection property. If $p$ satisfies the pairwise Markov property with respect to an UDAG $G$, then
\[
p(V)= \prod_i p(C_i|bd(C_i))= \prod_i \prod_{K \in cl((G_{C_i \cup bd(C_i)})^*)} \varphi(K)
\]
where $\varphi(K)$ is a non-negative function.
\end{theorem}

\begin{proof}

The first equality follows from the fact that $C_i \ci \cup_{j<i} C_j \setminus bd(C_i) | bd(C_i)$. To prove the second equality for $i=t$, note that $p$ satisfies the pairwise Markov property with respect to $G^m$, because $V$ is an ancestral set. Then, $p$ satisfies the global Markov property with respect to $G^m$ by Theorem \ref{the:eqPG}. Now, add undirected edges to $G^m$ until $bd(C_t)$ is a complete set, and call the resulting undirected graph $H$. Note that $p$ satisfies the global Markov property with respect to $H$. Note also that $(G_{C_t \cup bd(C_t)})^* = H_{C_t \cup bd(C_t)}$. Then, $p(C_t, bd(C_t))$ satisfies the global Markov property with respect to $H_{C_t \cup bd(C_t)}$ \citep[Proposition 2.2]{Frydenberg1990b}. This implies the second equality in the theorem because (i) $p(C_t, bd(C_t)) = \prod_{K \in cl((G_{C_i \cup bd(C_i)})^*)} \phi(K)$ \citep[Theorem 3.9]{Lauritzen1996}, (ii) $p(C_t | bd(C_t)) = p(C_t, bd(C_t)) / p(bd(C_t))$, and (iii) $bd(C_t)$ is a complete set in $(G_{C_t \cup bd(C_t)})^*$. Finally, note that $V \setminus C_t$ is an ancestral set and, thus, $p(V \setminus C_t)$ satisfies the pairwise Markov property with respect to $G_{V \setminus C_t}$. Then, repeating the reasoning above for $p(V \setminus C_t)$ and $G_{V \setminus C_t}$ proves the second equality in the theorem for all $i$.

\end{proof}

For instance, in the UDAG in Figure \ref{fig:example} we have that
\begin{align*}
p(V)=& p(A)p(B)p(D|B)p(CEFH|AD)\\
=&\varphi(A)\varphi(B)\varphi(DB)\varphi(CAD)\varphi(CFA)\varphi(CEF)\varphi(FH)\varphi(HD).
\end{align*}

Our causal interpretation of UDAGs is a generalization of the one proposed by \citet{LauritzenandRichardson2002} for LWF CGs. Specifically, they show that any probability distribution that satisfies the globally Markov property with respect to a LWF CG coincides with the equilibrium distribution of a dynamic system with feed-back loops. The proof consists in building a Gibbs sampler with the desired equilibrium distribution. The sampler samples the chain components in the order $C_1, \ldots, C_t$. Sampling a component $C_i$ consists in repeatedly updating the values of the variables $A \in C_i$ in random order according to the distribution $p(A | bd(C_i), C_i \setminus A)$ until equilibrium is reached. The interesting observation is that 
\[
p(A | bd(C_i), C_i \setminus A) = p(A | bd(C_i), ne(A))
\]
and thus the sampling process can be seen as a dynamic process with feed-back loops, since $A$ is dynamically affected by $ne(A)$ and vice versa. Thanks to the first equality in Theorem \ref{the:factorization2}, the causal interpretation just discussed can be generalized to UDAGs if
\[
p(A | bd(C_i), C_i \setminus A) = p(A | bd(C_i), pa(A), ne(A))
\]
that is, if $C_i \setminus (A \cup pa(A) \cup ne(A))$ carry no information about $A$ given $bd(C_i) \cup pa(A) \cup ne(A)$. One case where this may hold is when the causal relationships in the domain are constrained in their functional form, e.g. the effect is a function of the cause plus some noise, also known as additive noise model (ANM). If the function is non-linear, then it is unlikely that the cause can be expressed as an ANM of the effect \citep{Hoyeretal.2009,Petersetal.2014}. As a consequence, we expect any other variable to be hardly informative about the effect given the cause. Specifically, we expect $C_i \setminus (A \cup pa(A) \cup ne(A))$ to be hardly informative about $A$ given $bd(C_i) \cup pa(A) \cup ne(A)$, because we interpret the latter variables as the causes of $A$. So, our causal interpretation of UDAGs should approximately hold under the ANM assumption. We refer the reader to Appendix B for an algorithm for learning causal UDAGs and some preliminary experimental results.

\section{DISCUSSION}\label{sec:discussion}

We have introduced UDAGs, a new class of graphical models that unifies DAGs and UGs since it just forbids directed cycles and it allows up to two edges between any pair of nodes. We have presented local, pairwise and global Markov properties for UDAGs and proved their equivalence. We have also presented an equivalent factorization property. Finally, we have presented a causal interpretation of UDAGs. We refer the reader to the appendices for two learning algorithms for UDAGs and preliminary experimental results.

A natural question to tackle in the future is the characterization of Markov equivalent UDAGs. Although we have shown that UDAGs are a strict superclass of LWF CGs, it is unclear how much more expressive they are. Addressing this question is also of much interest. Finally, we are also interested in studying methods for parameterizing the factorization for UDAGs proposed here, as well as in improving the causal interpretation of UDAGs given here.

\bibliographystyle{plainnat}
\bibliography{unification8}

\newpage

\section*{APPENDIX A: ALGORITHM FOR LEARNING UDAGS}\label{sec:learning}

In this appendix, we describe an exact algorithm for learning UDAGs from data via answer set programming (ASP). The algorithm builds on the results in Section \ref{sec:algorithm} and it is essentially the same as the one developed by \cite{Sonntagetal.2015} for learning LWF CGs. ASP is a declarative constraint satisfaction paradigm that is well-suited for solving computationally hard combinatorial problems \citep{gelfond_1988,DBLP:journals/amai/Niemela99,DBLP:journals/ai/SimonsNS02}. ASP represents constraints in terms of first-order logical rules. Therefore, when using ASP, the first task is to model the problem at hand in terms of rules so that the set of solutions implicitly represented by the rules corresponds to the solutions of the original problem. One or multiple solutions to the original problem can then be obtained by invoking an off-the-shelf ASP solver on the constraint declaration. Each rule in the constraint declaration is of the form \verb|head :- body|. The head contains an atom, i.e. a fact. The body may contain several literals, i.e. negated and non-negated atoms. Intuitively, the rule is a justification to derive the head if the body is true. The body is true if its non-negated atoms can be derived, and its negated atoms cannot. A rule with only the head is an atom. A rule without the head is a hard-constraint, meaning that satisfying the body results in a contradiction. Soft-constraints are encoded as rules of the form \verb|:~ body. [W]|, meaning that satisfying the body results in a penalty of $W$ units. The ASP solver returns the solutions that meet the hard-constraints and minimize the total penalty due to the soft-constraints. A popular ASP solver is \verb|clingo| \citep{DBLP:journals/aicom/GebserKKOSS11}, whose underlying algorithms are based on state-of-the-art Boolean satisfiability solving techniques \citep{DBLP:series/faia/2009-185}.

Table \ref{tab:asp} shows the ASP encoding of our learning algorithm. The input to the algorithm is the set of independences in the probability distribution at hand, e.g. as determined from some available data. These are represented as a set of predicates \verb|ind(A,B,Z)| indicating that the nodes $A$ and $B$ are independent given the set of nodes $Z$. It is known that these pairwise independences (also called elementary triplets) uniquely identify the rest of independences in the distribution, or in a semi-graphoid for that matter \citep[Lemma 2.2]{Studeny2005}. The predicates \verb|node(A)| and \verb|set(Z)| represent that $A$ is the index of a node and $Z$ is the index of a set of nodes. The predicates \verb|line(A,B)| and \verb|arrow(A,B)| represent that there is an undirected and directed edge from the node $A$ to the node $B$. The rules 4 and 5 encode a non-deterministic guess of the edges, which means that the ASP solver will implicitly consider all possible UDAGs during the search, hence the exactness of the search. The rules 6 and 7 enforce the fact that undirected edges are symmetric and there is at most one directed edge between two nodes. The predicate \verb|ancestor(A,B)| represents that the node $A$ is an ancestor of the node $B$. The rules 8-10 enforce that there are no directed cycles. The predicates in the rules 11 and 12 represent whether a node $A$ is or is not in a set of nodes $Z$. The rules 13-23 encode the separation criterion for UDAGs as it was described in Section \ref{sec:algorithm}. Specifically, the predicate \verb|inU1(A,D,Z)| represents that there is a $Z$-active route from the node $A$ to the node $D$ that warrants the inclusion of $D$ in the set $U_1$. Similarly for the predicates \verb|inU2(A,D,Z)| and \verb|inU3(A,D,Z)|. The predicate \verb|act(A,B,Z)| in the rules 24 and 25 represents that there is a $Z$-active route between the node $A$ and the node $B$. The rule 26 enforces that each dependence in the input must correspond to an active route. The rules 27 and 28 represent a penalty of one unit per edge. Other penalty rules can be added similarly. By calling the ASP solver, the solver will essentially perform an exhaustive search over the space of UDAGs and return the sparsest minimal independence map.

It is worth noting that the algorithm in Table \ref{tab:asp} can easily be modified to learn DAGs and LWF CGs, which demonstrates the versatility of our approach. Specifically, learning DAGs can be performed by adding \verb|:- line(A,B).| Learning LWF CGs can be performed by adding \verb|:- line(A,B), arrow(A,B).|, \verb|:- line(A,B), arrow(B,A).|, and \verb|ancestor(A,B) :- line(A,B).|

Finally, preliminary experiments indicate that the algorithm in Table \ref{tab:asp} runs in acceptable time in a regular computer for up to seven nodes. To scale the learning process to larger domains, one may need to give up either the exactness or the assumption free nature of the algorithm here proposed.

\begin{table}
\caption{Algorithm for learning UDAGs.}\label{tab:asp}
\begin{center}
\tiny%\small
\fbox{
\begin{minipage}{0.95\textwidth}
\begin{verbatim}
% input predicate: ind(A,B,Z), the nodes A and B are independent given 
%                  the set of nodes Z

#const n=7.
node(1..n).
set(0..(2**n)-1).

% edges
{ line(A,B) } :- node(A), node(B), A != B.                             % rule 4
{ arrow(A,B) } :- node(A), node(B), A != B.
line(A,B) :- line(B,A).                                                % rule 6
:- arrow(A,B), arrow(B,A).

% directed acyclity
ancestor(A,B) :- arrow(A,B).                                           % rule 8
ancestor(A,B) :- ancestor(A,C), ancestor(C,B).
:- ancestor(A,B), arrow(B,A).                                  	

% set membership
in(A,Z) :- node(A), set(Z), 2**(A-1) & Z != 0.                         % rule 11
out(A,Z) :- node(A), set(Z), 2**(A-1) & Z = 0.

% rules
inU2(A,A,Z) :- node(A), set(Z), out(A,Z).                              % rule 13
inU2(A,D,Z) :- inU2(A,C,Z), arrow(D,C), out(D,Z).
inU2(A,D,Z) :- inU2(A,C,Z), line(D,C), out(D,Z).
inU1(A,D,Z) :- inU1(A,C,Z), arrow(C,D), out(D,Z).
inU1(A,D,Z) :- inU2(A,C,Z), arrow(C,D), out(D,Z).
inU1(A,D,Z) :- inU1(A,C,Z), line(C,D), out(D,Z).
inU3(A,D,Z) :- inU1(A,C,Z), arrow(C,D), in(D,Z).
inU3(A,D,Z) :- inU2(A,C,Z), arrow(C,D), in(D,Z).
inU3(A,D,Z) :- inU1(A,C,Z), line(C,D), in(D,Z).
inU3(A,D,Z) :- inU3(A,C,Z), line(C,D).
inU2(A,D,Z) :- inU3(A,C,Z), arrow(D,C), out(D,Z).

% active routes
act(A,B,Z) :- inU1(A,B,Z), A != B.                                     % rule 24
act(A,B,Z) :- inU2(A,B,Z), A != B.

% satisfy all the dependences
:- not ind(A,B,Z), not act(A,B,Z), node(A), node(B), set(Z), A != B,
   out(A,Z), out(B,Z).	                                                % rule 26

% minimize the number of lines/arrows
:~ line(A,B), A < B. [1,A,B,1]                                         % rule 27
:~ arrow(A,B). [1,A,B,2]

% show results
#show.
#show arrow(A,B) : arrow(A,B).
#show line(A,B) : line(A,B), A < B.
\end{verbatim}
\end{minipage}
}
\end{center}
\end{table}

\newpage

\section*{APPENDIX B: ALGORITHM FOR LEARNING CAUSAL UDAGS}

In Section \ref{sec:causal}, we have proposed a causal interpretation of UDAGs under the assumption that it approximately holds that
\[
p(A | bd(C_i), C_i \setminus A) = p(A | bd(C_i), pa(A), ne(A))
\]
that is, $C_i \setminus (A \cup pa(A) \cup ne(A))$ are hardly informative about $A$ given $bd(C_i) \cup pa(A) \cup ne(A)$. Note that any $(pa(A) \cup ne(A))$-active route between a node in $bd(C_i) \setminus pa(A)$ and $A$ reaches $A$ through a node in $(ch(A) \cap C_i) \setminus ne(A)$. Since we have assumed that $C_i \setminus (A \cup pa(A) \cup ne(A))$ are hardly informative about $A$ given $bd(C_i) \cup pa(A) \cup ne(A)$, so are $(ch(A) \cap C_i) \setminus ne(A)$. Then, we expect $bd(C_i) \setminus pa(A)$ to be hardly informative about $A$ given $pa(A) \cup ne(A)$ and, thus, it should approximately hold that
\[
p(A | bd(C_i), C_i \setminus A) = p(A | pa(A), ne(A)).
\]
So, learning a causal UDAG boils down to learning the causes $pa(A)$ and $ne(A)$ of each variable $A$ under the ANM constraint. Additive noise is a rather common assumption in causal discovery \citep{Petersetal.2017}, mainly because it produces tractable models which are useful for gaining insight into the system under study. Note also that linear structural equation models, which have extensively been studied for causal effect identification \citep{Pearl2009}, are ANMs.

\begin{table}
\caption{Algorithm for learning causal UDAGs.}\label{tab:algorithm}
\begin{center}
\begin{tabular}{|rl|}
\hline
&\\
& {\bf Input}: A dataset $D$ over $V$ with $M$ observations, and an integer $L$.\\
& {\bf Output}: A causal UDAG over $V$.\\
&\\
1 & Let $\mathcal{G}$ be a random sample of $L$ UDAGs over $V$\\
2 & For each $G$ in $\mathcal{G}$\\
3 & \hspace{0.3cm} For each $A \in V$\\
4 & \hspace{0.7cm} $\hat{f}(A|pa(A),ne(A)) = GP(A,pa(A), ne(A),D)$\\
5 & \hspace{0.7cm} For $m=1, \ldots, M$\\
6 & \hspace{1.1cm} $\hat{U}_A^m = A^m - \hat{f}(A^m|pa^m(A),ne^m(A))$\\
7 & \hspace{0.3cm} Let $D_U$ denote the dataset over $U$ created in the previous line\\
8 & \hspace{0.3cm} $pvalue(G)=HSIC(D_U)$\\
9 & Return the simplest model in the set $arg \, max_{G \in \mathcal{G}} \,\, pvalue(G)$\\
&\\
\hline
\end{tabular}
\end{center}
\end{table}

In this appendix, we propose an algorithm for learning causal UDAGs under the assumptions discussed above. The algorithm builds on the ideas by \cite{Hoyeretal.2009} and \cite{Petersetal.2011,Petersetal.2014}, who exploit the non-linearities in the data to identify the directions of the causal relationships. Specifically, consider two variables $A$ and $B$ that are causally related as $B = f(A) + U_B$. Assume that there is no confounding, selection bias or feed-back loop between $A$ and $B$, which implies that $A$ and $U_B$ are independent. \cite{Hoyeretal.2009} prove that if the function $f$ is non-linear, then the correct direction of the causal relationship between $A$ and $B$ is generally identifiable from observational data: $A$ and $U_B$ are independent for the correct direction, whereas $B$ and $U_A$ are dependent for the incorrect direction. This leads to the following causal discovery algorithm. If $A$ and $B$ are independent then they are not causally related because we have assumed no confounding, selection bias or feed-back loop. If they are dependent then first construct a non-linear regression of $B$ on $A$ to get an estimate $\hat{f}$ of $f$, then compute the error $\hat{U}_B=B-\hat{f}(A)$, and finally test whether $A$ and $\hat{U}_B$ are independent. If they are so then accept the model $A \ra B$. Repeat the procedure for the model $B \ra A$. When both models or no model is accepted, it may be indicative that the assumptions do not hold. \cite{Petersetal.2011,Petersetal.2014} generalize this idea to more than two variables: Given a DAG over some variables, first construct a non-linear regression of each node on its parents, then compute each node's error, and finally test whether these errors are mutually independent. If they are so then accept the DAG as the true causal model. We propose to generalize this idea even further: Given an UDAG over some variables, first construct a non-linear regression of each node on its parents and neighbors, then compute each node's error, and finally test whether these errors are mutually independent. If they do so then accept the UDAG as the true causal model. This leads to the learning algorithm in Table \ref{tab:algorithm}. It receives as input a dataset $D$ with $M$ observations over the random variables $V$, and an integer $L$. The algorithm consists in sampling $L$ random UDAGs over $V$ (line 1), scoring each of them with respect to $D$ (lines 2-8), and returning the best one (line 9). Scoring an UDAG $G$ starts in the lines 3-4 pretty much like the algorithms by \cite{Hoyeretal.2009} and \cite{Petersetal.2011,Petersetal.2014}, i.e. obtaining an estimate $\hat{f}$ of $f$ by constructing a non-linear regression of each node $A$ on $pa(A) \cup ne(A)$ using Gaussian processes (GPs) \citep{RasmussenandWilliams2005}. This estimate is used in the lines 5-7 to compute the errors. We use a superscript to indicate the value of a set of variables in a particular instance of $D$, i.e. $A^m$ and $pa^m(A)$ and $ne^m(A)$ represent the value of $A$ and its parents and neighbors in the $m$-th instance of $D$. Finally, the line 8 scores the whole model by testing the independence of the errors. The null hypothesis is joint independence. Specifically, the function $HSIC$ returns the $p$-value of the Hilbert Schmidt independence criterion, which is a kernel statistical test of independence \citep{Grettonetal.2007}. Note that several UDAGs may score the highest $p$-value, e.g. every supergraph of an UDAG with the highest $p$-value may also receive the highest $p$-value. Therefore, the line 9 applies the Occam's Razor principle and returns the simplest best UDAG. Using GPs and the HSIC test are choices shared with \cite{Hoyeretal.2009} and \cite{Petersetal.2011,Petersetal.2014}. Other choices are also possible. We have implemented our learning algorithm in \texttt{R}. We use the packages \texttt{kernlab} and \texttt{dHSIC} for the GPs and the HSIC test. For the GPs, we use the Gaussian kernel with automatic width estimation. For the HSIC test, we use the gamma distribution approximation to the null distribution of the test statistic.\footnote{Code available at \texttt{https://www.dropbox.com/s/fuz4eow8f66omda/UDAGs.R?dl=0}.}

The rest of this appendix reports on preliminary results obtained by running the learning algorithm above on the DWD dataset, which contains climate data from the German Weather Service and has been used before for benchmarking causal discovery algorithms \citep{Mooijetal.2016,Petersetal.2014}. We use the data provided by the first reference.\footnote{Data available at \texttt{https://www.dropbox.com/s/rs4q8oeutgqfcdn/D1.csv?dl=0}.} The data consists of 349 instances, each corresponding to a weather station in Germany. Each instance consists of measurements for six random variables. We use only four of them so that the number of UDAGs is manageable and, thus, our learning algorithm has a chance to test most if not all of them. Specifically, there are 34752 UDAGs over four nodes (543 DAGs times 64 UGs), and we let the algorithm sample 50000 in line 1. The four random variables that we consider are altitude ($A$), temperature ($T$), precipitation ($P$), and sunshine duration ($S$). The last three variables represent annual mean values over the years 1961-1990. \citeauthor{Mooijetal.2016} argue that the causal relationships $A \ra T$, $A \ra P$ and $A \ra S$ are true. Their arguments are meteorological, i.e. not based on the data. We confirmed that these decisions make sense according to Wikipedia (entries for the terms "rain" and "precipitation").

Figure \ref{fig:dwd} shows the causal UDAG learned, which is in fact a LWF CG. This UDAG is clearly preferred ($p$-value = 0.0007) over an UDAG with only the three ground truth relationships ($p$-value = 5.4e-40). The $p$-values should be interpreted with caution. Their relative values are informative. However, their absolute values may not, because low $p$-values may be the result of the GPs underfitting the data. Of course, low $p$-values may also be indicative of the inadequacy of the ANM assumption. It is also worth mentioning that the best and second best UDAGs found by our algorithm are Markov equivalent but they receive different $p$-values (0.0007 versus 0.0004) since they represent different causal models.

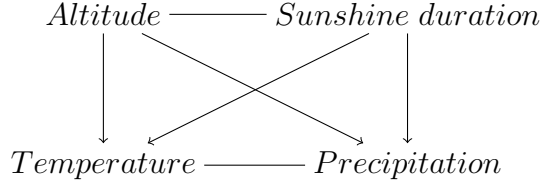
\begin{figure}
\centering
\begin{tabular}{c}
\begin{tikzpicture}[inner sep=1mm]
\node at (0,0) (A) {$Altitude$};
\node at (0,-2) (T) {$Temperature$};
\node at (4,-2) (P) {$Precipitation$};
\node at (4,0) (S) {$Sunshine \,\, duration$};
\path[->] (A) edge (T);
\path[->] (A) edge (P);
\path[->] (S) edge (T);
\path[->] (S) edge (P);
\path[-] (T) edge (P);
\path[-] (A) edge (S);
\end{tikzpicture}
\end{tabular}
\caption{Causal UDAG learned from the DWD dataset.}\label{fig:dwd}
\end{figure}

Now, we argue that the UDAG learned is plausible. The relationships $A \ra T$ and $A \ra P$ are confirmed by both \citeauthor{Mooijetal.2016} and Wikipedia. The relationship $S \ra T$ seems natural. We can also think of an intervention where we install new suns. We expect that the more suns the warmer. The relationship $A - S$ is confirmed by \citeauthor{Mooijetal.2016} due to selection bias: All the mountains in Germany are in the south and the south is typically sunnier. Note that this selection bias is at odds with the ANM assumption and, thus, one should not interpret the relationship $A - S$ as a feed-back loop. The relationship $A \ra S$ is also confirmed by \citeauthor{Mooijetal.2016} However, our algorithm is unable to learn an UDAG with a subgraph $A \ra S - A$, because removing the edge $A \ra S$ results in an UDAG with the same score, which is preferred by our algorithm because it is simpler. The relationship $T - P$ confirms the complex (including possibly feed-back) interplay between temperature and precipitation. According to Wikipedia, rain is produced by the condensation of atmospheric water vapor. Therefore, increasing water vapor in the air and/or decreasing the temperature are the main causes of precipitation. One way water vapor gets added to the air is due to increased temperature, causing evaporation from the surface of oceans and other water bodies. Moreover, precipitation typically causes a decrease in temperature, as rain drops form at high altitude where it is colder. The relationship $S \ra P$ is unconfirmed.

Finally, we mention some additional experiments that we plan to carry out. As discussed above, the best UDAG found is plausible. However, the second best UDAG receives a relatively close score (0.0007 versus 0.0004) but it is less plausible, as it includes the edge $P \ra A$. That the learning algorithm does not discriminate better these two models may be due to the inadequacy of the ANM assumption, but not necessarily. It may be that we have to choose carefully the width of the Gaussian kernel in the GPs, or consider other kernels, or consider other non-linear regressors, or consider the exact permutation-based null distribution of the HSIC test, or consider alternative hypothesis tests. We plan to study all these possibilities. In our experiments, the UDAG returned is actually a LWF CG. We expect UDAGs to reveal all their potential in larger and more complex domains. However, our brute-force learning algorithm does not scale well. That is why we plan to develop a greedy hill-climbing version of the algorithm that evaluates all the models that differ from the current one by one edge and then moves to the best of them. Note that we do not have to compute the score from scratch for each candidate model to evaluate, as at most two nodes are affected by a single edge modification.

\end{document}